\documentclass[twocolumn]{article} %DO NOT CHANGE THIS

\pdfoutput=1

\usepackage{times}  %Required
\usepackage{helvet} %Required
\usepackage{courier}  %Required
\usepackage{url}  %Required
\usepackage{graphicx}  %Required
\usepackage{balance}
\usepackage{algpseudocode}
\usepackage[Algorithm]{algorithm}
\usepackage{tabularx}
\usepackage{array}
\usepackage{hhline}
\usepackage{multirow}
\usepackage{subfig}
\usepackage{arydshln}
\usepackage{color}
\usepackage{amsfonts}
\usepackage{amsmath}
\usepackage{xspace}
\usepackage{amsthm}
\usepackage{multirow}
\usepackage{mathtools}
\usepackage{enumerate}
\newcommand\nin{\noindent}
\newcommand\emm[1]{{\it #1\/}}
\newcommand{\ignore}[1]{}

\newcommand\lan{\langle}
\newcommand\real{\mathbb{R}}
\newcommand\nat{\mathbb{N}}
\newcommand\ran{\rangle}

\newcommand*{\False}{\mathsf{False}}
\newcommand*{\True}{\mathsf{True}}

\newcommand{\ncond}[1]{\mathsf{S}_{#1}}
\newcommand{\npos}[1]{\mathsf{T}_{#1}}
\newcommand{\nneg}[1]{\mathsf{F}_{#1}}
\newcommand{\nwgt}[1]{\mathsf{W}_{#1}}
\newcommand{\nprd}[1]{\mathsf{P}_{#1}}

\newcommand{\dval}{wl}

\newcommand\gbc{{\sc GBC}\xspace}
\newcommand\gbr{{\sc GBR}\xspace}
\newcommand\dts{{\sc DTS}\xspace}
\newcommand\sat{{\sc SAT}\xspace}

\newcommand\smt{{\sc SMT}\xspace}
\newcommand\vgb{{\sc VeriGB}\xspace}

\newtheorem{theorem}{Theorem}[section]

\newtheorem{lemma}[theorem]{Lemma}
\newtheorem{definition}[theorem]{Definition}

\begin{document}
\sloppy

\title{Verifying Robustness of Gradient Boosted Models}
\author{Gil Einziger,
\ Maayan Goldstein,
\ Yaniv Sa'ar,
Itai Segall\\
{Nokia, Bell Labs}\\
gilein@bgu.ac.il,
\{maayan.goldstein, yaniv.saar, itai.segall\}@nokia-bell-labs.com
}
\date{}
\maketitle

\begin{abstract}
Gradient boosted models are a fundamental machine learning technique.
Robustness to small perturbations of the input is an important quality measure
for machine learning models, but the literature lacks a method to prove the
robustness of gradient boosted models.

This work introduces \vgb, a tool for quantifying the robustness of gradient
boosted models. \vgb encodes the model and the robustness property as an SMT
formula, which enables state of the art verification tools to prove the model's
robustness.
We extensively evaluate \vgb on publicly available datasets and demonstrate a
capability for verifying large models. Finally, we show that some model
configurations tend to be inherently more robust than others.
\end{abstract}

\section{Introduction}
Gradient boosted models are fundamental in machine learning and are among the
most popular techniques in practice. They are known to achieve good accuracy
with relatively small models, and are attractive in numerous domains ranging
from computer vision to
transportation~\cite{990517,DBLP:journals/corr/YangYLL15,Survey1,FREUND1997119,Chapelle:2010:YLR:3045754.3045756,transport}.
They are easy to use as they do not require normalization of input features, and
they support custom loss functions as well as classification and regression.
Finally, the method has a solid theoretical grounding~\cite{GBTheory}.

Machine learning models are often vulnerable to adversarial perturbations, which
may cause catastrophic failures (e.g., by misclassification of a traffic
sign). Specifically, Figure~\ref{fig:gtsrb_misses} exemplifies that gradient
boosted models are indeed vulnerable to such perturbations. Thus, identifying
which models are robust to such manipulations and which are not is critical.
Indeed, numerous works suggested training techniques that increase the
robustness~\cite{Survey2,Sun2007}.
However, there is currently no method to formally verify gradient
boosted models.
% 
% Furthermore, the community lacks the understanding of how 
Furthermore, it is not clear how
the configuration parameters of such models affect their robustness.
These knowledge gaps make it challenging to guarantee the reliability of
gradient boosted solutions.

\begin{figure}[tb]
	\begin{center}
		\includegraphics[width = 0.93\columnwidth]{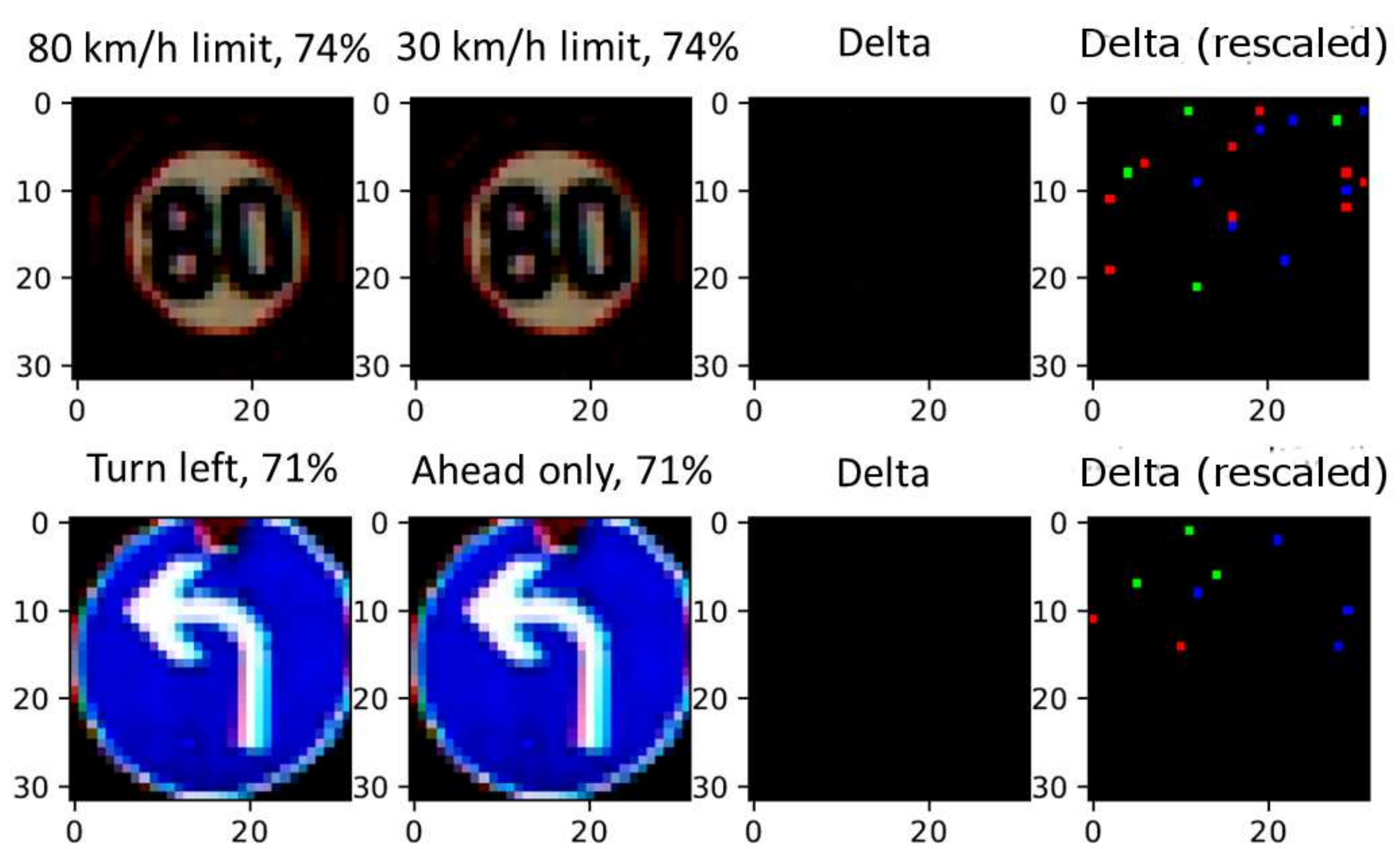}
	\end{center}
	\caption{Example of the lack of robustness in a gradient boosted model trained
	over a traffic signs dataset. In the first row, an ``80 km/h speed limit'' sign
	is misclassified as a ``30 km/h speed limit''. In the second row, a ``turn
	left'' sign is misclassified as ``ahead only''.
	% Observe that the changes are barely visible to the naked eye. The delta
	% between the images is given in the third column, while the forth column
	% highlights the modified pixels.
	Observe in the third column (delta, computed as the difference in pixel values
	of the two images) that the applied changes are barely visible to the naked eye
	(delta of +/-3 in the range of 256 values per pixel per color). The fourth column highlights the modified pixels.
	}
	\label{fig:gtsrb_misses}
\end{figure}

In the last couple of decades, formal methods successfully increased the
reliability of numerous software and hardware systems. Such success gave rise to
diverse verification methods such as model checking, termination analysis, and
abstract interpretation. Formal methods are especially appealing in situations
where the cost of mistakes is exceptionally high. Examples include mission-critical
solutions as well as mass-produced hardware.
Unfortunately, machine learning models are fundamentally different from
traditional software artifacts, and we cannot directly use existing
verification techniques for machine learning models.  %Thus, developing ways to
%verify the robustness of gradient boosted trees is necessary.
The research community already started addressing the problem for neural
network models~\cite{Pulina2010,KatzBDJK17,marta,Gehr2018AISA,mooly}. Here we
focus on an area that has not been covered so far -- verification of robustness of
gradient boosted models.

% \textbf{Contribution:}
% This work suggests novel techniques for proving the robustness of gradient
% boosted trees.  
The main contribution of this work is the \vgb tool for verifying the robustness
of gradient boosted models.
\vgb encapsulates novel and formally proven methods that translate such models,
and robustness properties into SMT formulas. Then, we feed these formulas to
a standard SMT solver, which proves the robustness or provides
a counter-example.
\vgb includes runtime optimizations that make the verification process
practical. We extensively evaluate it with public datasets and demonstrate
scalability for large and accurate models. Finally, we highlight that some model
configurations are fundamentally more robust than others.

The rest of this paper is organized as follows: In
Section~\ref{sec:preliminaries} we provide background on logic, decision trees,
and gradient boosted models. Next, in Section~\ref{sec:robustness}, we formally
define the robustness properties. The SMT formula representation of gradient
boosted models is given in Section~\ref{sec:encoding_gbt}, and that of the
robustness property in Section~\ref{sec:encoding_robustness}. Next,
Section~\ref{sec:optimization} suggests optimizations of these encodings
improving their runtime. Section~\ref{sec:eval} evaluates \vgb on
several publicly-available datasets, while Section~\ref{sec:related} surveys
related work.
We conclude in Section~\ref{sec:conclusion}, which discusses the implications of
our work and suggests directions for future research.

\section{Preliminaries}
\label{sec:preliminaries}

\subsection{Logic and Linear Arithmetic}
A \emm{propositional formula} is defined inductively as one of the following:
\emm{(i)} `True' and `False' constants (T and F).
\emm{(ii)} a variable $x_i \in \{x_1, \ldots, x_m\}$;
\emm{(iii)} if $\varphi$ and $\psi$ are propositional formulas then so are $\neg
\varphi$, $\varphi \vee \psi$, $\varphi \wedge \psi$, $\varphi \rightarrow
\psi$, $\varphi \leftrightarrow \psi$ (with their usual interpretation).
Given a propositional formula $\varphi$, the Boolean satisfiability problem
(\sat) determines whether there exists an assignment under which $\varphi$
evaluates to $\True$. %More details can be found in \cite{???}

\emph{Satisfiability Modulo Theories (\smt)} extends the Boolean \sat problem by
combining a variety of underlying theories \cite{BSST09}. We use the linear real
arithmetic theory, which extends the propositional fragment with all rational
number constants, and with the symbols: $\{+, -, \cdot, \leq, \geq\}$.
% 
% We chose the real theory as it is known to be more efficient than the natural
% number theory.
% 
A formula $\varphi$ (be that an \smt or \sat instance) is said to be
\emm{satisfiable}, if $\varphi$ evaluates to $\True$ for some assignment
$\vec{x} \in \real^{m}$. If there is no such assignment, we say that $\varphi$
is \emm{unsatisfiable}.

\subsection{Decision Trees}
Decision trees are functions that receive an assignment $\vec{x} \in \real^{m}$
and return a value. Formally, a \emm{decision tree structure} (\dts) $D = \lan
N, I, L\ran$ is defined as follows:
% \ysnote{replace $n = \lan cnd, pos, neg\ran$ with $n = \lan
% \texttt{cnd}, n_{_{T}}, n_{_{F}}\ran$, or $n = \lan
% \mathsf{cnd}, n^{T}, n^{F}\ran$?}
\begin{itemize}
	\item $N = \{n_1, \ldots, n_k\}$: is the set of nodes in the tree, and $n_1$
	is defined to be the \emm{root node} of the tree.
	\item $I \subseteq N$: is the subset of internal nodes in the tree. An
	\emm{internal node} is a triplet $n = \lan \ncond{n}, \npos{n}, \nneg{n}\ran$,
	where $\ncond{n}$ is a condition expressing the decision of node $n$ (an \smt
	formula), and $\npos{n}\in N$ (resp., $\nneg{n} \in N$) is the target successor
	node when the condition evaluates to $\True$ (resp., $\False$).
	\item $L = N \setminus I$: is the subset of leaf nodes in the tree, i.e. nodes
	for which there is no successor. A \emm{leaf node} $n = \lan \nwgt{n} \ran$
	also has a weight $\nwgt{n} \in \real$.
\end{itemize}
%
% Given a node $n\in I$, $\ncond{n}$ is the condition of node $n$, $\npos{n}$
% (resp. $\nneg{n}$) is positive child (resp., negative child) of node $n$. Given
% a node $n\in L$, $\nwgt{n}$ is the weight of node $n$.
%
% Intuitively, $\ncond{}$, $\npos{}$, and $\nneg{}$ are dictionaries that
% associate to every $n\in I$ a condition $\ncond{n}$, a positive child
% $\npos{n}\in N$, and a negative child $\nneg{n}\in N$, respectively.
%
Intuitively, $\ncond{}$ (resp., $\npos{}$ and $\nneg{}$) is a dictionary
that associates to every $n\in I$ a condition $\ncond{n}$ (resp., a positive
child $\npos{n}\in N$ and a negative child $\nneg{n}\in N$).
$\nwgt{}$ is a dictionary that associates to every $n\in L$ a weight $\nwgt{n}
\in \real$.

A \dts $D$ is said to be \emm{well-formed} if, and only if, every node $n \in
N$ has exactly one predecessor node, except for the root node that has no
predecessor. In a well-formed tree, we denote by $\nprd{n}$ the predecessor of
node $n\in N$.
Given an input vector $\vec x \in \real^m$, the \emm{valuation of a \dts $D$ on
$\vec x$} is a function $\hat{D}: \real^m \rightarrow \real$. Tree $D$ is
traversed according to $\vec{x}$, ending in a leaf node $n\in L$, and function
$\hat{D}(\vec x)$ is the weight of that node, i.e. $\nwgt{n}\in \real$.

\subsection{Gradient Boosted Trees}
Gradient boosted regression is an ensemble technique that constructs a strong
learner by iteratively adding weak learners (typically decision
trees)~\cite{GBTheory}.
Formally, a \emm{Gradient Boosted Regressor (\gbr)} is a sequence of $r$
decision trees $R = \lan D_1, \ldots, D_r\ran$.
% 
% We denote the functions of tree $D_i$ by $cnd_i(n)$, $pos_i(n)$, $neg_i(n)$,
% $wgt_i(n)$, $prd_i(n)$.
% 
Given an input vector $\vec{x} \in \real^m$, the \emm{valuation of a \gbr $R$}
is the sum of valuations of its $r$ decision trees. That is, $\hat{R}(\vec{x}) =
\sum^r_{i=1} \hat{D}_i(\vec{x})$.

Gradient boosted classification is a tree ensemble technique that constructs a
strong learner per each class (again, by iteratively adding weak learners), to
assign a class for a given input. Let $c$ be the number of classes.
Formally, a \emm{Gradient Boosted Classifier} (\gbc) $C = \lan R_1, \ldots,
R_c\ran$ is a sequence of $c$ gradient boosted regressors, where regressor
$R_j = \lan D_1^j, \ldots, D_r^j \ran$.
%
% We denote by $cnd_i^j(n)$, $pos_i^j(n)$, $neg_i^j(n)$, $wgt_i^j(n)$,
% $prd_i^j(n)$, $\hat T_i^j$ the respective functions for tree $T_i^j$, and by
% $\hat R_j$ the respective function for regressor $R_j$.
%
% We denote by $\hat R_j$ the respective function for regressor $R_j$.
%
Given an input vector $\vec{x} \in \real^m$, the \emm{valuation of $C$},
valuates all $c$ regressors over $\vec{x}$ and returns the class associated with
the maximal value, namely: $\hat{C}(\vec{x}) = \arg\max_j (\hat{R}_j(\vec{x}))$.
We assume that there is an association between each input vector and a single class\footnote{In
cases where multiple regressors return the same maximal value we can break the
symmetry using their indices.}.
%
% \ysnote{
% The interested reader is referred to~\cite{???} for the techniques of training
% such models.
% --
% remove? or change to ``there are million way to do the trainning\ldots
% the reader is\ldots'':
% }

\section{Robustness of Machine Learning Models}
\label{sec:robustness}
% 
%\ysnote{remove ``adversarial''?}
% 
Robustness means that small perturbations in the input have little effect on the
outcome. That is, for classifiers the classification remains the same, and for
regressors, the change in valuation is bounded. This section formally defines
robustness properties, in a similar manner
to~\cite{Pulina2012,mooly,KatzBDJK17,Robustness-Moosavi}.

Consider a regression model $R$, and let $\hat{R}(\vec{x})$ be the valuation
function of $R$ for an input $\vec{x} \in \real^m$. We define \emm{local
adversarial $(\epsilon, \delta)$-robustness for an input $\vec{x}$}, as follows:
\begin{definition}[local adversarial robustness of regressors]
	\label{def:local_robust_regression}
	A regression model $R$ is said to be $(\epsilon, \delta)$-robust for an input
	$\vec{x}$, if for every input $\vec{x}'$ such that $||\vec{x} - \vec{x}'||_p <
	\epsilon$, the output is bound by $\delta$, i.e.,
	$|\hat{R}(\vec{x}) - \hat{R}(\vec{x}')| \leq \delta$.
\end{definition}
\nin Here, $||\vec{x}-\vec{x}'||_p$ is used to specify the distance between two
vectors $\vec{x}$ and $\vec{x}'$ according to some norm $p$. For example, one may
compute the distance between two images as the maximal difference between pairs
of corresponding pixels (i.e., $p = \infty$), or the sum of these differences 
(i.e., $p = 1$).
Throughout this paper we use norm $p = \infty$, but our techniques are
applicable to any norm that is linear to the input.   

Next, consider a classification model $C$ and let $\hat{C}(\vec{x})$ be the
valuation function of $C$ for an input $\vec{x}\in\real^m$. We define
\emm{local adversarial $\epsilon$-robustness for an input $\vec{x}$} as follows:
\begin{definition}[local adversarial robustness of classifiers]
	\label{def:local_robust_classifictation}
	A classification model $C$ is said to be $\epsilon$-robust for an input
	$\vec{x}$, if for every input $\vec{x}'$ such that $||\vec{x} - \vec{x}'||_p <
	\epsilon$, the output does not change its classification, i.e.,
	$\hat{C}(\vec{x}) = \hat{C}(\vec{x}')$.
\end{definition}

\ignore{Moved to the encoding section. 
	To avoid the complexity of checking universal quantification over $\vec{x}'$, we
	negate the definition of local robustness. Namely, given a regression model (resp.,
	classification model) we check whether the expression $||\vec{x} - \vec{x}'||_p < \epsilon\
	\wedge\ |\hat{R}(\vec{x}) - \hat{R}(\vec{x}')| \geq \delta$ (resp., $||\vec{x} -
	\vec{x}'||_p < \epsilon\ \wedge\ \hat{C}(\vec{x}) \neq \hat{C}(\vec{x}')$) is
	satisfiable. Observe, that an assignment that satisfies this expression is a
	counter-example that disproves local adversarial robustness of the given input $\vec{x}$.
	On the other hand, if there is no such assignment, we proved local adversarial robustness.
}

% The above definitions of local adversarial robustness aim to certify a given
% input, which is useful to gain confidence for a specific input. However, they do
% not guarantee much regarding the model itself. To this end, we extend these
% definitions to capture the behavior over all inputs, by universally quantifying
% input $\vec{x}$. Since quantifying input $\vec{x}$ is too strict in the context
% of the statistical model, we turn to an alternative formalism which refers to a set
% of inputs $A$. %$A = \{\vec{x}_1, \ldots, \vec{x}_k\}$.
%
The above definitions aim to certify a given input but do not guarantee much
regarding the model itself. Therefore, we extend these definitions to capture
the behavior over a set of inputs $A$. %$A = \{\vec{x}_1, \ldots, \vec{x}_k\}$.
%
% \ysnote{I think that the sentence regarding statistical quantification is
% interesting...}
%
We define \emm{$\rho$-universal adversarial $(\epsilon, \delta)$-robustness on
a set of inputs $A$}, as follows:
\begin{definition}[universal adversarial robustness of regressors]
	\label{lemma:universal_robust_regression}
	A regression model $R$ is said to be $\rho$-universally $(\epsilon, \delta)$-robust
	%(for $\epsilon, \delta, \rho > 0$)
	over the set of inputs $A$, if it is $(\epsilon, \delta)$-robust for at least
	$\rho \cdot |A|$ inputs in $A$.
\end{definition}
% \nin Note that if $\rho=1$ and $A$ is the universal set that contains all
% inputs, then this definition collapses to the strict universal quantification
% overall input $\vec{x}$.

Finally, we extend the classifier definition of local $\epsilon$-robustness, and
define \emm{$\rho$-universal adversarial $\epsilon$-robustness on a set of
inputs $A$}, as follows:
\begin{definition}[universal adversarial robustness of classifiers]
	\label{def:universal_robust_classifictation}
	A classification model $C$ is said to be $\rho$-universally $\epsilon$-robust
	%(for $\epsilon, \delta, \rho > 0$)
	over the set of inputs $A$, if it is
	$\epsilon$-robust for at least $\rho \cdot |A|$ inputs in $A$.
\end{definition}

Definition~\ref{lemma:universal_robust_regression} and
Definition~\ref{def:universal_robust_classifictation} capture the universal
adversarial robustness properties for regressors and classifiers.
%
%For regressors and classifiers, 
%
The parameter $\epsilon$ determines the allowed perturbation change, that is,
how much an attacker can change the input. For regressors, we also require the
parameter $\delta$ that defines the acceptable change in the output, while for
classifiers we require that the classification stays the same.
Finally, the parameter $\rho$ measures the portion of robust inputs. In
Section~\ref{sec:eval}, we evaluate the $\rho$ values of varying models instead
of selecting a $\rho$ value in advance.
%  If the model is robust only for a fraction $q$ of input vectors, we say that
% it is $(q,\epsilon)$-robust.
% During our experimental evaluation presented in Section~\ref{sec:eval} we show
% the percentage $\rho$ of the inputs that were found robust out of a randomly
% selected subset of test set inputs $A$ provided as part of the analyzed
% datasets.

% In the experimental evaluation presented in Section~\ref{sec:eval}, we consider
% a set $A$ that is a randomly selected subset of inputs taken from the test-set.

\section{Encodings of Gradient Boosted Models}
\label{sec:encoding_gbt}
This section explains the encoding of gradient boosted models into SMT formulas.
We start by translating a single path in a decision tree and then work our way
up until we end up with a formula for the entire model. 
% Next, we use standard tools (e.g., Z3~\cite{Z3}) to find a counter-example or
% to prove the robustness.

\subsection{Encoding of Decision Trees}
Given a well-formed \dts $D = \lan N, I, L \ran$ and a leaf $l \in L$, we define
\emm{$path(l)$} to be the set of nodes on the path in the tree between the leaf
node $l$ and the root node $n_1$ (including both nodes). We define the
\emm{encoding of leaf $l$ in tree $D$} to be the formula $\pi(l)$ as follows:
\begin{equation*}
\pi(l):\bigwedge\limits_{n \in path(l)\setminus\{n_1\}}  
\left( 
\begin{aligned}
&\npos{\nprd{n}}=n \rightarrow \hspace{0.5em} \ncond{\nprd{n}} \\[0.25em]  
&\nneg{\nprd{n}}=n \rightarrow \neg \ncond{\nprd{n}}
\end{aligned}
\wedge \right)
\wedge \left(\dval = \nwgt{l}\right) 
\end{equation*}
The encoding $\pi(l)$ restricts the \emm{decision tree valuation variable}
$\dval$ to be the weight of the leaf ($\dval = \nwgt{l}$), and for each node $n$
in the path except for the root, if node $n$ is the positive child of its parent
($\npos{\nprd{n}} = n$) then the parent condition should hold
($\ncond{\nprd{n}}$), and if node $n$ is the negative child of its parent
($\nneg{\nprd{n}} = n$) then the negation of the parent condition should hold
($\neg \ncond{\nprd{n}}$).

%\ysnote{note that ``reaches leaf $l$'' is not neccessary for the first side of
% the proof, and only for proving the other side it is required (note that there
% could be several leaves with the same weight).}
%%% actually I don't think that we need to bi-implication in this proof...
%%% 
%%% The encoding of $\pi(l)$ over inputs $\vec x \in \real^m$ and $\dval \in \real$
%%% evaluates to $\True$ if and only if the valuation of decision tree $D$ on $\vec
%%% x$ reaches leaf $l$ and outputs $\hat{D}(\vec x) = \nwgt{l}$.
\begin{lemma}[leaf encoding]
\label{lemma:correctness_leaf}
Let $\hat D$ be the valuation function of the well-formed tree $D$.
If $\pi(l)$ evaluates to $\True$, then there exists a truth assignment $\vec x
\in \real^m$, $\dval \in \real$ such that $\hat{D}(\vec x)$ reaches leaf $l$ ,
and $\hat{D}(\vec x) = \nwgt{l} = \dval$.
\end{lemma}
\begin{proof}
Assume that the leaf encoding $\pi(l)$ evaluates to $\True$, then there exists a
truth assignment $\vec x \in \real^m$, $\dval \in \real$.
%
% Since the tree is well-formed and following the definition of $path(l)$, we know
% that every internal node $n'\in path(l) \cap I$ is a predecessor of some node
% $n\in path(l)$, i.e., $n' = \nprd{n}$.
% %
% If $n$ is the positive successor of $n'$, then $\npos{n'} = n$. Therefore,
% $\npos{\nprd{n}} = n$ holds, implying that $\ncond{n'}$ holds for $\vec{x}$ as
% well. Thus, when the valuation of $\hat{D}(\vec x)$ traverses tree $D$ and
% reaches node $n'$, we know that it indeed turns to the positive child.
%
Since the tree is well-formed and following the definition of $path(l)$, we know
that every internal node $n'\in path(l) \cap I$ is a predecessor of some node
$n\in path(l)$, i.e., $n' = \nprd{n}$.
If $n$ is the positive successor of $n'$, then $(\npos{\nprd{n}} = n)$ holds,
implying that $\ncond{n'}$ holds for $\vec{x}$ as well. Thus, when the valuation
of $\hat{D}(\vec x)$ traverses tree $D$ and reaches node $n'$, we know that it
indeed turns to the positive child.
% and $\npos{n'}=n$.
%
The same reasoning applies to the negative successor of $n'$.
By applying this reasoning recursively from the root node, 
we show that the traversal of the valuation reaches leaf $l$, and outputs
$\hat{D}(\vec x) = \nwgt{l} = \dval$.
\end{proof}

Given \dts $D = \lan N, I, L \ran$, we now define the \emm{encoding of tree
$D$} to be the formula $\Pi(D)$ as follows:
% $\Pi(D) = \bigvee_{l \in L} \pi(l)$
\begin{equation*}
\Pi(D):\quad \bigvee_{l \in L} \pi(l)
\end{equation*}
Namely, $\Pi(D)$ is a disjunction of formulas, where each disjunct represents a
concrete path to one of the leaves in $D$ and its respective valuation.
\begin{lemma}[tree encoding]
\label{lemma:correctness_tree}
Let $\hat D$ be the valuation function of the well-formed tree $D$.
If $\Pi(D)$ evaluates to $\True$, then there exists a truth assignment $\vec x
\in \real^m$, $\dval \in \real$, and a single leaf $l\in L$ for which
$\hat{D}(\vec x)$ reaches $l$ and outputs $\hat{D}(\vec x) = \nwgt{l} = \dval$.
\end{lemma}
\begin{proof}
Assume that the tree encoding $\Pi(D)$ evaluates to $\True$, then there exists
a truth assignment $\vec x \in \real^m$, $\dval \in \real$.
Clearly, at least one clause in $\Pi(D)$ evaluates to $\True$.
Since tree $D$ is well formed, at most one clause in $\Pi(D)$ evaluates to
$\True$, otherwise there exists an internal node in the path $n\in path(l) \cap
I$ for which $\ncond{n}$ is inconsistent over $\vec x$.
Therefore, there exists exactly one clause in $\Pi(D)$ that evaluates to
$\True$, and exactly one leaf $l\in L$ for which $\pi(l)$ evaluates to $\True$.
%
% Following Lemma~\ref{lemma:correctness_leaf}, if $\pi(l)$ evaluates to $\True$,
% then $\hat{D}(\vec x)$ reaches this leaf $l$, and outputs $\hat{D}(\vec x) =
% \nwgt{l} = \dval$.
%
If $\pi(l)$ evaluates to $\True$, then following the same reasoning of
Lemma~\ref{lemma:correctness_leaf}, the truth assignment $\vec x \in \real^m$,
$\dval \in \real$ reaches leaf $l$ and outputs $\hat{D}(\vec x) = \nwgt{l} =
\dval$.
\end{proof}

\subsection{Encoding of Gradient Boosted Trees}
Given \gbr $R = \lan D_1, \ldots, D_r\ran$ and following
Lemma~\ref{lemma:correctness_leaf}, and Lemma~\ref{lemma:correctness_tree}, we
define the \emm{encoding of regressor $R$} to be the formula $\Upsilon(R)$ as
follows:
\begin{equation*}
\Upsilon(R):\quad  \Big(\bigwedge\limits_{i=1}^{r}
\Pi(D_i)\Big) \wedge out=\sum\limits_{i=1}^{r} \dval_i
\end{equation*}
Intuitively, $\Upsilon(R)$ consists of two parts: \emm{(i)} the conjunction of
all tree encodings, ensuring that the decision tree valuation variables of each
tree $\dval_1, \ldots, \dval_r$ are restricted to their respective tree
valuations; and \emm{(ii)} a restriction of the \emm{regressor valuation
variable} $out$ to be the sum of all decision tree valuation variables $\dval_1,
\ldots, \dval_r$.
Therefore, encoding $\Upsilon(R)$ characterizes regressor $R$. 
\begin{theorem}[regressor encoding]
\label{thm:regresor_correctness}
Let $\hat{R}$ be the valuation function of regressor $R$.
If $\Upsilon(R)$ evaluates to $\True$, then there exist a truth assignment
$\vec x \in \real^m$, $out \in \real$, such that $\hat{R}(\vec x) = out$.
\end{theorem}
\begin{proof}
The proof follows from the definitions and Lemma~\ref{lemma:correctness_tree}.
%The proof follows immediately from Lemma~\ref{lemma:correctness_tree} that says
%that if the tree encoding evaluates to $\True$, then the tree encoding has a
%truth assignment. This assigns $out$ with $\hat{R}$.  
\end{proof}

Given \gbc $C = \lan R_1, \ldots, R_c \ran$ and following
Theorem~\ref{thm:regresor_correctness}, we define the \emm{encoding of
classifier $C$} to be the formula $\Gamma(C)$ as follows:
\begin{equation*}
\Gamma(C):\quad \bigwedge\limits_{j=1}^c \Upsilon(R_j)  
\wedge \bigvee\limits_{j=1}^c \Big( arg = j \leftrightarrow
\bigwedge\limits_{k=1}^c out_j > out_k \Big)
\end{equation*}
Intuitively, $\Gamma(C)$ consists of two parts: \emm{(i)} the conjunction of all
regressor encodings, ensuring that the regressor valuation variables $out_1,
\ldots, out_r$ are restricted to their respective regressor valuations; and
\emm{(ii)} a restriction of the \emm{classifier valuation variable} $arg$ to be
the maximal regressor valuation (i.e., operator $\arg\max$).
Therefore, $\Gamma(C)$ charactarizes classifier $C$.
\begin{theorem}[classifier encoding]
\label{thm:classifier_correctness}
Let $\hat{C}$ be the valuation function of classifier $C$.
If $\Gamma(C)$ evaluates to $\True$, then there exist a truth assignment $\vec x
\in \real^m$, $arg \in \{1,\ldots,c\}$, such that $\hat{C}(\vec x) = arg$.
\end{theorem}
\begin{proof}
The proof follows from the definitions, theorem, and lemmas above.
\end{proof}

%%%%%%%%%%%%%%%%%%%%%%%%%%%%%%%%%%%%%%%%%%%%%%%%%%%%%%%%%%%%%%%%%%%%%%%%%%%%%%%
%%%%%%%%%%%%%%%%%%%%%%%%%%%%%%%%%%%%%%%%%%%%%%%%%%%%%%%%%%%%%%%%%%%%%%%%%%%%%%%
%%%%%%%%%%%%%%%%%%%%%%%%%%%%%%%%%%%%%%%%%%%%%%%%%%%%%%%%%%%%%%%%%%%%%%%%%%%%%%%

% \subsection{Encoding of Local Robustness Properties}
\section{Encodings of Local Robustness Properties}
\label{sec:encoding_robustness}
% In Section~\ref{sec:encoding_gbt} we encoded the structure of the gradient
% boosted models.
In this section, we encode the local robustness properties defined in
Section~\ref{sec:robustness}.
Recall that a regression model (resp., classification model) satisifies local
adversarial robustness for an input $\vec{x}$
(Definitions~\ref{def:local_robust_regression},
and~\ref{def:local_robust_classifictation}), if for all $\vec{x}'$, if
$||\vec{x} - \vec{x}'||_p < \epsilon$, then the difference between the valuation
of $\vec{x}$, and that of $\vec{x}'$ is bound (resp., we get the same
classification for $\vec{x}$, and for $\vec{x}'$).

Our goal is to find whether there exists an assignment to $\vec{x}'$ that
satisfies both the model encoding, and the negation of the local adversarial
robustness property.
%
% Given a regression model (resp., classification model) we check whether
% expression $||\vec{x} - \vec{x}'||_p < \epsilon\ \wedge\ |\hat{R}(\vec{x}) -
% \hat{R}(\vec{x}')| \geq \delta$ (resp., $||\vec{x} - \vec{x}'||_p < \epsilon\
% \wedge\ \hat{C}(\vec{x}) \neq \hat{C}(\vec{x}')$) is satisfiable.
%
An assignment $\vec{x}'$ that satisfies both conjuncts constitutes a
\emm{counter-example} that disproves local adversarial robustness of the given
input $\vec{x}$. Alternatively, local adversarial robustness holds if there is
no such assignment.

Given an input $\vec{x}$, and $\epsilon, \delta \geq 0$, we define the
\emm{encoding of local adversarial robustness} to be a formula $\Phi$ as
follows:
\begin{align*}
\Phi: \quad \phi \ \wedge \ 
\bigwedge\limits_{i=1}^m \ 
\begin{cases}
|x_i - x'_i| \leq \epsilon, & x_i\in\real \\
x'_i \in \{v\in\nat \ : \ |x_i - v| \leq \epsilon\}, &
x_i\in\nat
\end{cases}
\end{align*}
%
% \begin{align*}
% \Phi: 
% &\bigwedge\limits_{i=1}^m \ 
% \begin{cases}
% x_i - \epsilon \leq x'_i \leq x_i + \epsilon, & x_i\in\real \\
% x'_i \in \{v\in\nat: x_i - \epsilon \leq v \leq x_i + \epsilon\}, &
% x_i\in\nat
% \end{cases} \\[0.25em]
% ַַ&\ \wedge \quad \phi
% \end{align*}
%
Where $\phi$ is $|\hat{R}(\vec{x}) - \hat{R}(\vec{x}')| \geq \delta$ for
regression model, and $\phi$ is $\hat{C}(\vec{x}) \neq \hat{C}(\vec{x}')$ for
classification model.
Note that the second range of conjuncts in the expression, characterizes the
allowed pertubations ($||\vec{x} - \vec{x}'||_p < \epsilon$) for norm
$p=\infty$, which is handled differently for real, and integer features.
%
% Specifically, since we use real number theory, we need to restrict the value of
% such features explicitly. Integer features are frequent, e.g., in the MNIST and
% GTSRB datasets features are pixels that are represented by integers in the range
% of [0,255].

\section{Optimizations}
\label{sec:optimization}
While the construction in Sections~\ref{sec:encoding_gbt} and
\ref{sec:encoding_robustness} is sound and complete, it is not always the most
efficient one. Thus, we now provide two optimizations based on eliminating
redundant clauses that cannot be satisfied, and on parallelizing the
verification process.

\subsection{Pruning}
``Pruning'' is a somewhat overloaded term. In the context of machine learning,
pruning typically refers to the process of removing sections of decision trees
that have little impact on the outcome, thus reducing over-fitting. In
the model-checking community, pruning is the process of trimming unreachable
parts of the search space, thus helping the model checker focus its
search.

Our approach combines these two notions. Namely, we remove all unsatisfiable
leaf clauses with respect to the robustness parameter ($\epsilon$), which allows
for faster calculation.
Formally, given \dts $D=\lan N, I, L \ran$ and property $\Phi$, we define the
\emm{$\Phi$-pruned encoding of leaf $l$ in tree $D$} to be:
\begin{equation*}
	\pi^{\Phi}(l) = 
	\begin{cases}
		\pi(l), & \pi(l)\wedge\Phi \textrm{ is satisfiable} \\
		\False, & \pi(l)\wedge\Phi \textrm{ is unsatisfiable}
	\end{cases}
\end{equation*} 
Note that pruning can be applied to diverse properties, but this work is focused
on the robustness property.

Next, we define the corresponding $\Upsilon^{\Phi}$ (resp., $\Gamma^{\Phi}$) to
be the \emm{$\Phi$-pruned encoding of regressor $R$} (resp., \emm{$\Phi$-pruned
	encoding of classifier $C$}), which replaces each occurrence of leaf encoding
$\pi(l)$ with its pruned version $\pi^{\Phi}(l)$.
% in the regressor encoding $\Upsilon(R)$ (resp., classifier encoding
% $\Gamma(C)$).
%
% Given a GBR $R$, a GBC $C$, and a property $\Phi$, t
The following theorem establishes the correctness of $\Phi$-pruning:
\begin{theorem}[safe pruning]\ \\[-1em] 
\begin{enumerate}[1.]
	\item \textbf{Regressor:} the conjunction $\Upsilon(R) \wedge \Phi$
	is satisfiable, if and only if, the conjunction $\Upsilon^{\Phi}(R) \wedge \Phi$
	is satisfiable.
	\item \textbf{Classifier:} the conjunction $\Gamma(C) \wedge \Phi$ is
	satisfiable, if and only if, the conjunction $\Gamma^{\Phi}(C) \wedge \Phi$ is
	satisfiable.
	%
	% \item The conjunction of the regressor encoding $\Upsilon(R)$ and property
	% $\Phi$ is satisfiable, if and only if, the conjuction of the $\Phi$-pruned
	% regressor encoding $\Upsilon^{\Phi}(R)$ and property $\Phi$ is satisfiable.
	% %
	% \item The conjunction of the classifier encoding $\Gamma(C)$ and property
	% $\Phi$ is satisfiable, if and only if, the conjuction of the $\Phi$-pruned
	% classifier encoding $\Gamma^{\Phi}(C)$ and property $\Phi$ is satisfiable.
\end{enumerate}
\end{theorem}
\begin{proof}
	The proofs follow immediately from the associativity property of propositional
	logic.
\end{proof}

In principle, we may use an SMT solver to check the satisfiability of $\pi(l)
\wedge \Phi$  for each leaf, in each tree.  In practice, we reduce the
dependence on SMT solvers and increase scalability by evaluating the robustness
property during the encoding of the tree, where each internal node condition
constraints a single feature $x_i$.
For norm $p=\infty$, the leaf valuation $\pi(l)$ is satisfiable, if and only if
for every node $n\in path(l)$ that refers to feature $x_i$, $|x_i - x'_i| \leq
\epsilon$.
For norm $p=1$, a necessary condition for the satisfiability of $\pi(l)$, is
that all features of $\vec{x}'$, $\sum^m_{i=1} d_i \leq \epsilon$, where:
\begin{equation*}
	d_i =
	\begin{cases}
		|x_i - x'_i|, & x_i \ \textrm {appears in } path(l) \\
		0, & x_i \ \textrm{does not appear in } path(l)
	\end{cases}
\end{equation*}
The pruning process removes paths where the given vector $\vec{x}'$ is ``far''
from the required thresholds by more than $\epsilon$, where the notion of
distance is determined by the norm.
%
% is according to the distance metric of choice. That is, under the restriction
% of at most $\epsilon$ perturbation to $\vec{z}$ (in the relevant norm), one can
% never reach this leaf.

\subsection{Parallelization}
It is difficult to parallelize general SMT formulas efficiently. To
increase scalability, we design our encoding in a manner that allows for
parallel evaluation of gradient boosted classifiers.
We do so by checking the robustness property separately for each
class index.
If all parallel evaluations are found robust, then the robustness property
holds.
Otherwise, there exists an assignment $\vec{x}$, and an index $q$, such that
the robustness property does not hold, and $\vec{x}$ is a counter-example. The
thread of class $q$ would discover this case and abort all other threads.

Formally, we do the following:%
\begin{equation*}
	\begin{aligned}
		& \quad \forall \vec{x}' &: \ &
		||\vec{x}-\vec{x}'||_p < \epsilon \rightarrow
		\hat{C}(\vec{x}) = \hat{C}(\vec{x}')  \\
		\Leftrightarrow & \ \neg \exists \vec{x}' &: \ &
		||\vec{x}-\vec{x}'||_p < \epsilon \wedge
		\hat{C}(\vec{x}) \neq \hat{C}(\vec{x}')\\
		\Leftrightarrow & \ \neg \exists \vec{x}'&: \ &
		||\vec{x}-\vec{x}'||_p < \epsilon \wedge
		\hat{C}(\vec{x}) \neq \sideset{}{_j}{\arg\max} \left(\hat{R}_j(\vec{x}')
		\right) \\
		\Leftrightarrow & \ \neg \exists \vec{x}'&: \ &
		||\vec{x}-\vec{x}'||_p < \epsilon \wedge
		\exists q: 
		\hat{R}_q(\vec{x}') > \hat{R}_{\hat{C}(\vec{x})}(\vec{x}')\\
		\Leftrightarrow & \ \neg \exists \vec{x}',q\hspace*{-0.75em}&: \ &
		||\vec{x}-\vec{x}'||_p < \epsilon \wedge
		\hat{R}_q(\vec{x}') > \hat{R}_{\hat{C}(\vec{x})}(\vec{x}') 
	\end{aligned}  
\end{equation*} 
Where the parameter $q$ is within $[1,c]$, and each thread verifies a different
value of $q$. 
For example, if an input is classified as class $a$, we invoke $c-1$ threads for
classes $\{1, \ldots, c\} \setminus \{a\}$, where each thread tries to verify
robustness with respect to a specific class.

\section{Evaluation}
\label{sec:eval}
We now introduce \vgb (Verifier of Gradient Boosted models), which implements
our approach in Python. \vgb utilizes Z3~\cite{z3} as the underlying SMT
solver.
We used the \texttt{sklearn}~\cite{sklearn_api} and \texttt{numpy}~\cite{numpy}
packages to train models.
We conducted the experiments on a VM with $36$ cores, a CPU speed of $2.4$ GHz,
a total of $150$ GB memory, and the Ubuntu 16.04 operating system. 
The VM is hosted by a designated server with two Intel Xeon E5-2680v2 processors
(each processor is made of 28 cores at 2.4 Ghz), $260$ GB memory, and Red Hat
Enterprise Linux Server $7.3$ operating system.
For tractability, we capped the runtime of verifying the local robustness
property by $10$ minutes.
We evaluated \vgb using the following three datasets:
\begin{enumerate}[1.]
	\item The \emm{House Sales in King County (HSKC)} dataset containing
	$22K$ observations of houses sold in between May 2014 and May
	2015 in King County, USA~\cite{Housing}. Each observation has $19$ house
	features, as well as the sale price. 
	\item The \emm{Modified National Institute of Standards and Technology
		(MNIST)} dataset containing $70K$ images of handwritten digits \cite{mnist}.
	The images are of size $28$ x $28$ pixels, each with a grayscale value ranging
	from $0$ to $255$. The images are classified into $10$ classes, one for each
	digit.
	\item The \emph{German Traffic Sign Recognition Benchmark (GTSRB)} dataset
	containing $50K$ colored images of traffic signs \cite{gtsrb}.
	The images are of size $32$ x $32$ pixels, each with three values (RGB) ranging
	from $0$ to $255$. The images are classified into $43$ classes, one for each
	traffic sign.
\end{enumerate}

\subsection{Regressor Evaluation}
We start by demonstrating \vgb's scalability to large gradient boosted
regression models using the HSKC dataset. 
We trained regressors varying the learning rates in $\{0.1,0.2,0.3\}$, the
number of trees between $50$ and $500$, and the tree depth in $\{3,5,8,10\}$.
All models have a similar score\footnote{The term score refers to the
coefficient of determination $R^2$ of the prediction.} that varies between
$0.84$ and $0.88$. Then we randomly selected $200$ observations and 
evaluated the $\rho$-universal $(\epsilon,\delta)$-robustness property with
an $\epsilon$ value of $160$ sq/ft, for the $6$ numerical features that refer to
square footage, and a $\delta$ value of $100$K\$ in the price. 
Note that there were no timeouts (where it took the SMT solver more than $10$
minutes to reach a decision) for models with less than $500$ trees, and even with $500$
trees we had only $16\%$ timeouts.

%\ysnote{missing the number of examples (set $A$)}

%
\begin{figure}[tb]
	\begin{center}
		\includegraphics[width = \columnwidth]{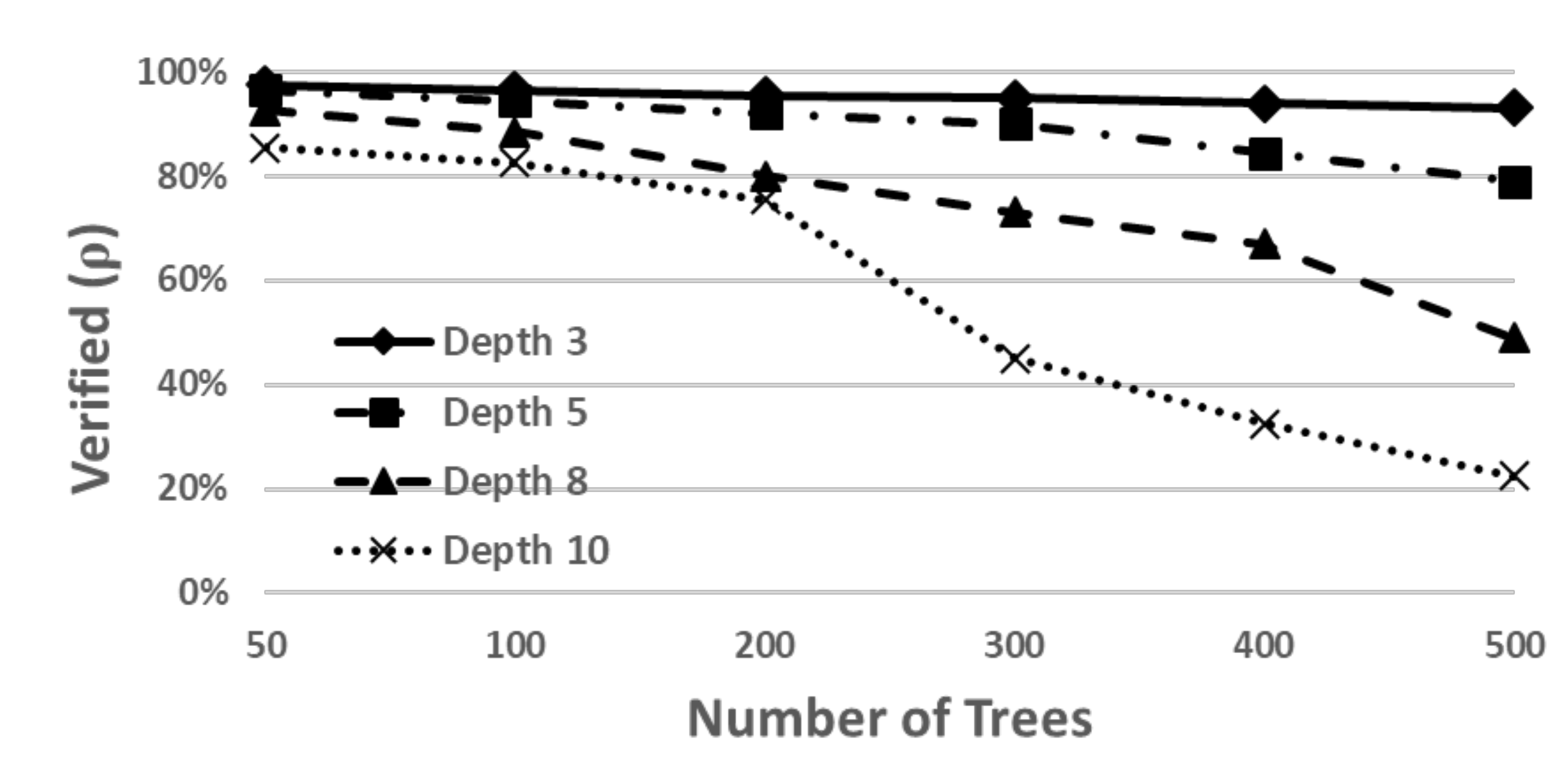}
	\end{center}
	\caption{Universal robustness eveluation for $\epsilon=160$ sq/ft, and
		$\delta=100$K\$, and regressors with a similar score.
		Illustrating the attainable portion of robust observations $\rho$, varying the number of trees and the tree depth. }
	\label{fig:housing_graph}
\end{figure}
Figure~\ref{fig:housing_graph} illustrates the results for a learning rate of
$0.1$, while the results for other learning rates are similar.
Notice that \emm{(i)} robustness degrades as the number of trees 
increases. \emm{(ii)} robustness seems to be negatively correlated with the tree
depth. That is, a model trained with a tree depth of $3$ is more robust than a
depth of $5$, which is more robust than $8$ and $10$.
%

% In all models \vgb found examples that are not locally robust. 
% For instance, it found a counter-example for a model with 100 trees and
% tree depth of 3, where the total sq/ft of a house increased from $3810$ to
% $3885$, and its basement sq/ft decreased from $1130$ to $1125$, while the sale
% price of the house increased by $176$K\$.

\subsection{Classifier Evaluation}
Next, we demonstrate \vgb's capability to verify the robustness of accurate
classification models.
We trained gradient boosted models for the MNIST and GTSRB datasets with a
learning rate of $0.1$. We varied the number of trees between $20$ and
$100$, and the maximal tree depth between $3$ and $20$. The accuracy of said models
varied between $87.9\%$ and $97.3\%$ for MNIST, and between $90\%$ and $96.86\%$
for GTSRB.
We evaluated the $\rho$-universal $\epsilon$-robustness property with
$\epsilon$ values of $1$, $3$, and $5$. We randomly selected $20$ images from each
class in the training set ($200$ images for MNIST, and $860$ images for GTSRB).

The illustration in Figure~\ref{fig:gtsrb_misses} is an artifact of this
evaluation. Recall, that it shows two examples where the local adversarial
robustness property does not hold for $\epsilon=3$ for a model trained for the
GTSRB dataset. In the first example, an ``80'' km/h speed limit sign is
misclassified as a ``30'' km/h limit. In the second example, a ``turn left'' sign is misclassified as an ``ahead only'' sign. 
\begin{figure}[tb]
	\begin{center}
		\includegraphics[width = 0.93\columnwidth]{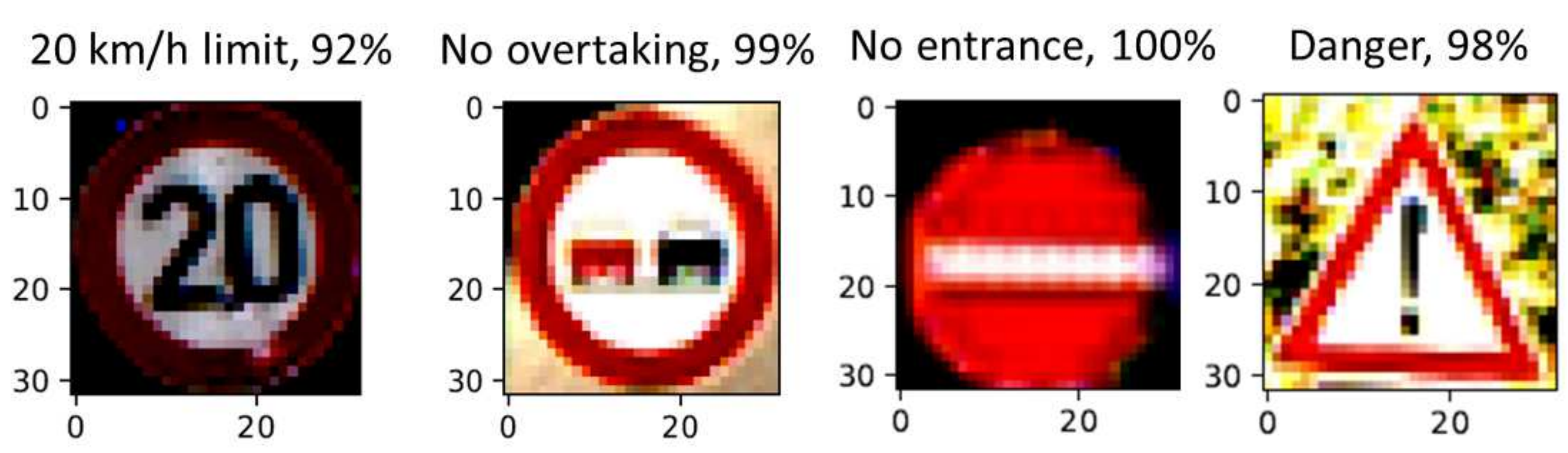}
	\end{center}
	\caption{Examples of GTSRB images that satisfy the local adversarial robustness
		property for $\epsilon =3$.}
	\label{fig:gtsrb_verified}
\end{figure}
Alternatively, Figure~\ref{fig:gtsrb_verified} shows examples of signs that
do satisfy the local adversarial robustness property for $\epsilon =3$. That is,
their classification would not change under any adversarial perturbation
that changes each pixel's RGB values by at most $3$.

Figure~\ref{fig:mnist_verified} shows examples of handwritten digits
that satisfy the local adversarial robustness property for $\epsilon =3$,
for models trained for the MNIST dataset.
\begin{figure}[tb]
	\begin{center}
		\includegraphics[width = 0.93\columnwidth]{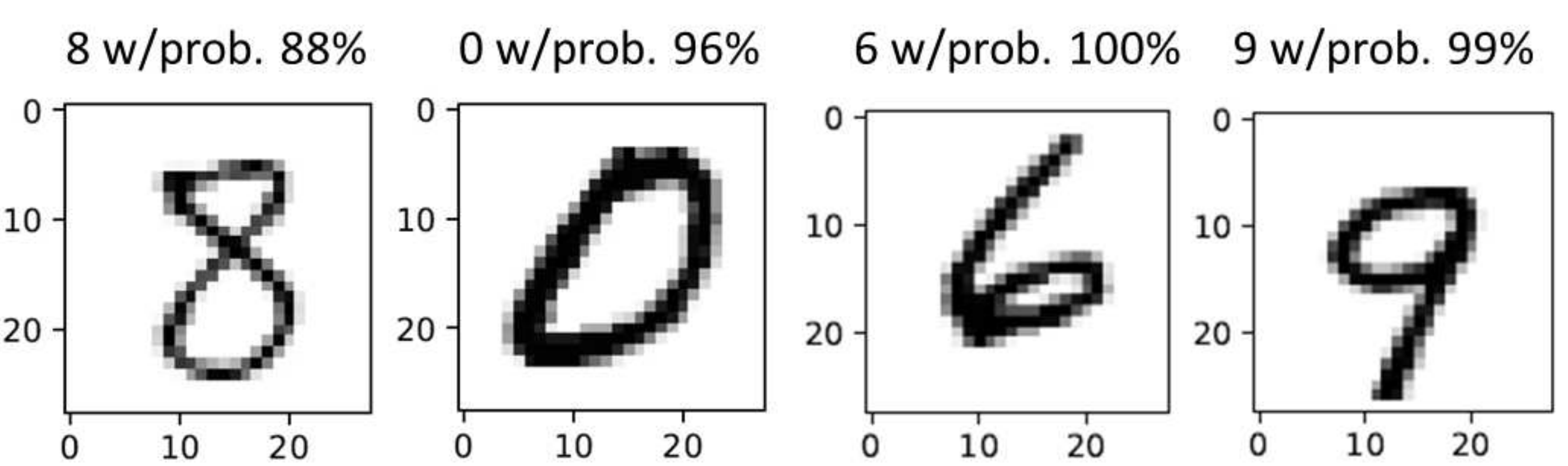}
	\end{center}
	\caption{Examples of MNIST images that satisfy the local adversarial robustness
		property for $\epsilon =3$. }
	\label{fig:mnist_verified}
\end{figure}
Alternatively, Figure~\ref{fig:mnist_misses} shows two examples where the local
adversarial robustness property does not hold. In the first example,
an image of ``1'' is misclassified as ``7''. The second image is 
misclassified as ``0'' instead of ``5'' under very slight perturbation.
These modifications are almost invisible to a human eye.
Note that the model's confidence does not indicate robustness. E.g., in the
first example the image has $95\%$ confidence to be classified as $1$, while
after applying the perturbation, it has $90\%$ confidence while being
misclassified as $7$.
\begin{figure}[tb]
	\begin{center}
		\includegraphics[width = 0.93\columnwidth]{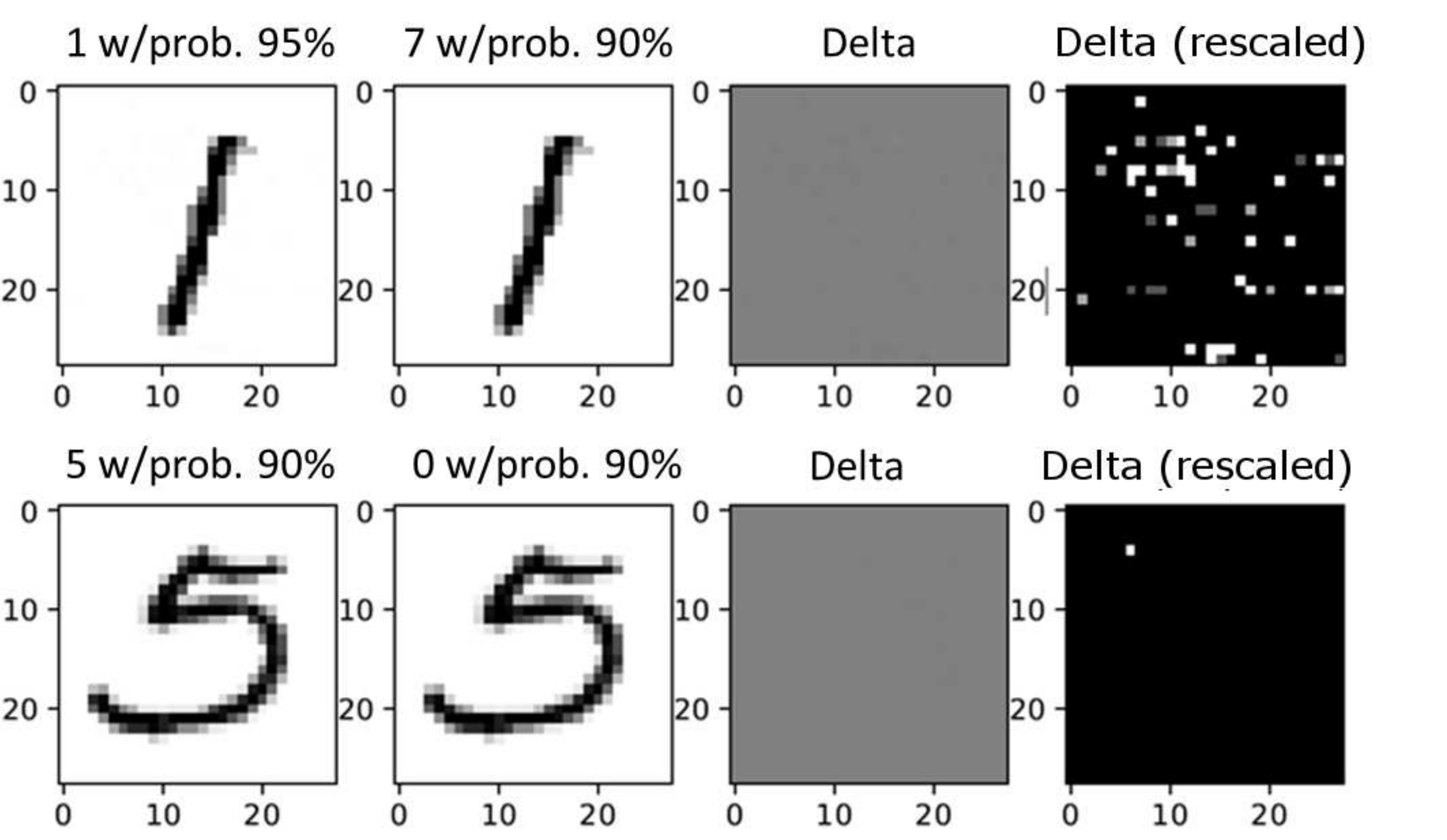}
	\end{center}
	\caption{Examples of MNIST images that do not satisfy local adversarial
	robustness for $\epsilon=3$. In the first row, an image of ``1'' is
	misclassified as ``7''. In the second row, an image of ``5'' is misclassified
	as ``0''.
	% Observe that the changes are barely visible to the naked eye. The delta
	% between the images is given in the third column, while the forth column
	% highlights the modified pixels.
	Observe in the third column (delta) that the applied changes are barely
	visible to the naked eye (delta of +/-3 in the range of 256 values per pixel
	per color). The fourth column highlights the modified pixels.
	}
	\label{fig:mnist_misses}
\end{figure}

\subsubsection{Scalability and limitations}
Table~\ref{fig:mnist-robustness-tbl} summarizes the results for selected models
trained for the MNIST dataset. In the table, the abbreviations ``T/O'' and
``C/E'' stand for the portion of timeouts and counter-examples, respectively.
Note that for a fixed tree depth, the portion of counter-examples found is
negatively correlated with the model's accuracy.
This is also true for a fixed number of trees.
In this example, large models with $100$ trees and high tree depth already
exhibit a non-negligible portion of timeouts,  indicating the limitations of
\vgb. Despite that fact, it successfully verifies highly accurate models for
the MNIST dataset.
We run similar experiments on models trained for the GTSRB dataset, with roughly
similar results. Unlike MNIST, the portion of timeouts was only $1\%$, even for
large models. As with MNIST, the portion of counter-examples
varies between $10\%$ and $22\%$. Finally, the
ratio of robust images varies between $78\%$ and $88\%$.
\begin{table*}[tb]
	\centering
	%\setlength{\tabcolsep}{5pt}
	%\renewcommand{\arraystretch}{1.25}
	%
	% preferred 
	%\renewcommand\arraystretch{1.2}
	%%%%%%%%%%
	%% allowed shrinking.
	%\fontsize{9.5pt}{10.5pt}
	%\selectfont
	
	%Maayan: actually the instuctions say the minium is 9.0 1.0
	\fontsize{8.5pt}{10.0pt}
	\selectfont
	%%%%%%%%%%
	\begin{tabular}{|c|c|c|c|c|c|c|c|c|c|c|c|}
		\hline
		\multirow{2}{*}{ Depth} & \multirow{2}{*}{Trees} &  \multirow{2}{*}{Accuracy}
		& \multicolumn{3}{c|}{$\epsilon=1$} & \multicolumn{3}{c|}{$\epsilon=3$} & \multicolumn{3}{c|}{$\epsilon=5$} \\
		\cline{4-12}
		
		& & & Verified ($\rho$)& T/O &  C/E & Verified ($\rho$) & T/O
		&  C/E & Verified ($\rho$) & T/O &  C/E\\
		\hline\hline
		3 & 20 & 87.9 & 16.5\% & 0\% & 83.5\%    & 10\% & 0\% & 90\% & 10\% & 0\% &
		90\%\\
		3 & 50 & 92.4 & 24\% & 0\% & 76\% & 24\% & 0\% & 79\% & 21\% & 0\% & 79\%    \\
		
		3 & 100 & 94.4 & 39.5\% & 0.5\% & 60\% & 31.5\% & 0.5\% & 68\% & 31.5\% &
		0.5\% & 68\%     \\    
		\hdashline[2.5pt/5pt]
		8 & 20 & 94.8 & 39.5\% & 0\% & 60.5\%    & 21\% & 0\% & 79\% & 21\% & 0\%        & 79\%\\
		
		8 & 50 & 96.4 & 53.5\% & 6\% & 40.5\%    & 40\% & 9.5\% & 50.5\% & 42.5\% & 7\%
		& 50.5\%\\
		8 & 100 & 97 & 29\% & 41.5\% & 29.5\% & 20\% & 45\% & 35\% & 22\% & 43.5\% &
		34.5\%
		\\
		\hdashline[2.5pt/5pt]
		10 & 20 & 95.6 &  39.5\% & 0\% & 60.5\% & 25\% & 0\% & 75\% & 25\% & 0\% &
		75\%
		\\
		
		10 & 50 &  96.7 &  53\% & 8.5\% & 38.5\% & 39.6\% & 10.6\% & 49.8\% & 46\% &
		8.5\% & 45.5\%
		\\
		10 & 100     & 97.3    & 15\% & 60\% & 25\%    & 10.5\% & 62.5\% & 27\% & 11.5\% &
		62.5\% & 26\%  \\
		\hline
	\end{tabular}
	\caption{MNIST dataset: Evaluating the attainable portion of robust
	observations $\rho$, for models with varying number of trees, tree depth, and
	$\epsilon$. The abbreviations ``T/O'' and ``C/E'' stand for the portion of
	timeouts and counter-examples, respectively.}
	\label{fig:mnist-robustness-tbl}
\end{table*}

%It took \vgb between $16-24$ minutes to verify the robustness of each model on
%the selected $200$ MNIST images. During that time, it solved between $1615$ and
%$1775$ constraints. Thus, we conclude that \vgb can verify high accuracy
%gradient boosted models for the MNIST dataset, within a reasonable amount of
%time.

\paragraph{The effect of model structure on robustness}

As a side-effect of this research, we noticed that certain configuration 
parameters tend to result in more robust models. Hereafter, we briefly
discuss our observations. Table~\ref{table:varios_trees} summarizes selected results for
models with a similar accuracy which is achieved by varying the number of trees,
and the tree depth.
% 1
% As can be observed, models with smaller tree depth are more robust (i.e., have a
% higher $\rho$ value).
As can be observed, models with smaller tree depth have a higher $\rho$ value.
% 2
The results show that the tree depth has a potentially large impact on robustness.
% 3
That is, increasing the tree depth leads to less robust results.
Notice that tree depth similarly affects the robustness of regression models, as
is clearly indicated in Figure~\ref{fig:housing_graph}.

It is interesting to mention, that tree depth also plays a role in the
over-fitting problem of gradient boosted models. Models with large tree depth
are more likely to suffer from over-fitting~\cite{hastie01statisticallearning}.
In our context, a small tree depth yields better robustness and is also easier
to verify,  making \vgb attractive for practical use cases.

% larger depth tends
%to cause over-fitting.  %Further, the number of trees and trees
%depth increase complexity and make it harder for \vgb to reach a decision,
% resulting in more timeouts. Thus, it is encouraging to learn that a small tree depth yields more robust results in
%practice.
%  This insight requires a thorough investigation, which we leave for future
% work.

\begin{table}[tb]
	\centering
	%\renewcommand\arraystretch{1.2}
	%Maayan: actually the instuctions say the minium is 9.0 1.0
	\fontsize{9pt}{10pt}
	\selectfont
	\begin{tabular}{|c|c|c|c|c|c|}
		\hline 
		Depth & Trees  & Accuracy & Verified ($\rho$) & T/O &  C/E \\
		\hline\hline
		4 & 100 & 95.6 & 53\% & 3\% & 44\%    \\
		5 & 65 & 95.7 & 52\% & 1\% & 47\%    \\
		7 & 40 & 95.8 & 52\% & 0.5\% & 47.5\%    \\
		10 & 20 & 95.6 & 39.5\% & 0\% & 60.5\%    \\
		20 &    18    & 95.8    & 27.5\% & 0.0\%    & 72.5\% \\    
		%40 & 15 & 95.66 & 41\% & 8.5\% & 50.5\%    \\
		\hline
	\end{tabular}
	\caption{MNIST dataset: Impact of boosted model's architecture on the
		attainable $\rho$ for the universal adversarial robustness property with
		$\epsilon =1$.}
	\label{table:varios_trees}
\end{table}

\section{Related Work}
\label{sec:related}
Reliability and security are of increasing interest by the research
community.
%\cite{Space_bugs}.
Numerous works demonstrate the (lack of) security
of popular machine learning models~\cite{Biggio,biggio1,Biggio2014}.
Others show
methods to generate adversarial inputs to such models~\cite{Zhou}. Thus,
certifying that specific models are robust to adversarial inputs is an important
research challenge. Indeed~\cite{mooly,KatzBDJK17,Gehr2018AISA}, introduced
methods for verifying robustness for various types of neural network models.
The robustness of gradient boosted models is also of interest, but existing works
are focused on empirical evaluation~\cite{Survey2}, or on training methods that
increase robustness~\cite{Sun2007}, while our work is the first to
certify gradient boosted models with formal and rigorous analysis.
% of gradient boosted
%models. 
% Specifically, our work is based on formal analysis of the model. %
% rather than on heuristics and empirical experimentation.

%Verification and validation are especially crucial in autonomous systems, when
%errors are difficult to fix, or can have dramatic results. The authors of
%\cite{Brat} identifies space exploration as an excellent motivator.
%Specifically, in a mission to Mars simple software bugs can cause catastrophic
%consequences ~\cite{space_bugs}.  Similar arguments follow for more terrestrial
%fields such as mass production, or autonomous vehicles.  The authors suggest a
%combination of advanced verification techniques such as static analysis, model
%checking, compositional verification, automated test generation, and certified
%code synthesis to ensure the maximum reliability. Sadly, standard formal methods
%are not directly applicable within the context of machine learning, due to
%numerous reasons; they leverage on software engineering practices, on the
%structure of the code, and on the developers' understanding of its inner
%workings. In contrast, ML code is usually not even human-readable, and even the
%very developers that trained the model do not entirely understand its operation.
%Note that machine learning techniques are increasingly central in autonomous
%systems which makes their verification crucial.

Since our work is the first and only work that verifies gradient boosted
model, we survey existing works that verify other machine learning
models.
In~\cite{marta}, the authors suggest an SMT based approach for verifying
feed-forward multi-layer neural networks. They use a white box approach to
analyze the neural network layer by layer and also apply a set of methods to
discover adversarial inputs.
Note that gradient boosted models are fundamentally different from neural
networks and thus their method does not extend to such models.
In~\cite{KatzBDJK17}, the authors describe a Simplex based verification
technique, that is extended to handle the non-convex Rectified Linear Unit
(ReLU) activation functions. Such activation is fundamental in modern neural
networks and is not expressible with linear programming. The main disadvantage
of that approach is its inability to scale up to large networks with
% hundreds of
thousands of ReLU nodes.

Alternatively, $AI^2$~\cite{Gehr2018AISA} uses ``abstract transformers'' to
overcome the difficulty of formally describing non-linear activation functions.
Safety properties such as robustness are then proved based on the abstract
interpretation. The over-approximation that is inherent in the technique allows
for scalable analysis. However, since they use abstractions, the
counter-examples provided are not always real counter-examples, and thus a
refinement process is required to end up with a concrete counter-example.

Finally, the authors of~\cite{mooly} adapt Boolean satisfiability to verify the
robustness of Binarized Neural Networks (BNN). Specifically, they apply a
counter-example-guided search procedure to check for robustness to adversarial
perturbations. They verified BNN models for the MNIST dataset. In comparison,
\vgb verifies slightly more accurate gradient boosted models for the same
dataset.
Similarly, in~\cite{Ehlers17} the authors propose a method for
verification of feed-forward neural networks. Their approach leverages
piece-wise linear activation functions. The main idea is to use a linear
approximation of the overall network behavior that can then be solved by SMT or
ILP.

\section{Conclusions and Future Work}
\label{sec:conclusion}
Our work is the first to verify robustness to adversarial perturbations for
gradient boosted models. Such models are among the most popular machine learning
techniques in practice. Our work introduces a model verification tool called
\vgb that transforms the challenge of certifying gradient boosted regression and
classification models into the task of checking the satisfiability of an SMT
formula that describes the model and the required robustness property. This
novel encoding is an important contribution of our work and includes formal
correctness proofs as well as performance optimizations. Once we have such
an (optimized) SMT formula, we check its
satisfiability with a standard solver. The solver either proves the robustness
property or provides a counter-example.

We extensively evaluated \vgb, with $3$ public datasets, and demonstrated its scalability to large and accurate models with hundreds of trees.
Our evaluation shows that the classification's confidence does not
provide a good indication of robustness.
Further, it indicates that models with a small tree depth tend to be more robust
even if the overall accuracy is similar.
Such models are also known to suffer less from over-fitting. We believe that
there may be an implicit correlation between robustness and good generalization,
and leave further investigation to future work.
Additionally, the counter-examples generated by \vgb may be leveraged in the
training phase of the gradient boosted models to optimize their robustness. However,
we leave such usage for future work.

%%%%%%%%%%
%% allowed shrinking.
\fontsize{9.5pt}{10.5pt}
\selectfont
%%%%%%%%%%
\balance
\bibliography{verigb}
\bibliographystyle{abbrv}
\end{document}